\documentclass[smallcondensed]{svjour3} 
\usepackage{amsmath,amssymb}
\usepackage{graphicx}
\usepackage[ruled]{algorithm2e}
\usepackage{amsmath,graphicx,color,doi}
\usepackage{algorithm2e}
\usepackage{multirow}
\usepackage{subfigure}
\usepackage{cite}
\usepackage{pgf,tikz,pgfplots}
\usetikzlibrary{arrows}
\usepackage{url}
\urldef{\mailsa}\path|firstname.lastname@polimi.it,|

\newcommand{\gi}[1]{}
\newcommand{\di}[1]{}
\newcommand{\ce}[1]{}
\newcommand{\ma}[1]{}

\def\l{\mathcal{L}}
\def\x{\mathbf{x}}

\def\m{\mathbf{\mu}}

\def\SNRPhi{\textnormal{SNR}(\phi_0 \to \phi_1)}

\def\KLPhi{\textnormal{sKL}(\phi_0,\phi_1)}
\def\LGUNO{\log \left( \phi_1(\x) \right)}
\def\LGZERO{\log \left( \phi_0(\x) \right)}

\def\mX{\mathcal{X}}

\def\mR{\mathbb{R}}
\def\change{\phi_0 \to \phi_1}
\def\ezero{\underset{\x \sim \phi_0}{E}}
\def\eone{\underset{\x \sim \phi_1}{E}}
\def\varzero{\underset{\x \sim \phi_0}{\textnormal{var}}}
\def\varone{\underset{\x \sim \phi_1}{\textnormal{var}}}

\newcommand{\mb}[1]{\mathbf{#1}}

\definecolor{myred}{rgb}{0.8,0.2,0.2}%
\definecolor{myblue}{rgb}{0.2,0.2,0.8}%
\definecolor{mygreen}{rgb}{0.2,0.8,0.2}%
\smartqed

\begin{document}

	\title{Change Detection in Multivariate Datastreams: Likelihood and 
		Detectability Loss}

\author{Cesare Alippi \and Giacomo Boracchi \and Diego Carrera \and Manuel 
	Roveri}
%

\institute{Dipartimento di Elettronica, Informazione e Bioingegneria,\\ 
	Politecnico di Milano, Milano, Italy\\
	\mailsa\\}
	\date{}
	\maketitle

	\begin{abstract}We address the problem of detecting changes in multivariate datastreams, and we investigate the intrinsic difficulty that change-detection methods have to face when the data dimension scales. In particular, we consider a general approach where changes are detected by comparing the distribution of the log-likelihood of the datastream over different time windows. Despite the fact that this approach constitutes the frame of several change-detection methods, its effectiveness when data dimension scales has never been investigated, which is indeed the goal of our paper.
		
	We show that the magnitude of the change can be naturally measured by the symmetric Kullback-Leibler divergence between the pre- and post-change distributions, and that the detectability of a change of a given magnitude worsens when the data dimension increases. This problem, which we refer to as \emph{detectability loss}, is due to the linear relationship between the variance of the log-likelihood and the data dimension. We analytically derive the detectability loss on Gaussian-distributed datastreams, and empirically demonstrate that this problem holds also on real-world datasets and that can be harmful even at low data-dimensions (say, 10).

	\end{abstract}

	\section{Introduction}\label{sec:Introduction}
	Change detection, namely the problem of detecting changes in probability distribution of a 
	process generating a datastream, has been widely investigated on scalar (i.e. univariate) data.
	Perhaps, the reason beyond the univariate assumption is that change-detection 
	tests (CDTs) were originally developed for quality-control 
	applications~\cite{Basseville1993}, and much fewer works address the problem 
	of detecting changes in multivariate datastreams.
	
	A straightforward extension to the multivariate case would be to independently 
	inspect each component of the datastream with a scalar 
	CDT~\cite{Tartakovsky2006}, but this does not clearly provide a truly 
	multivariate solution, e.g., it is unable to detect changes affecting the 
	correlation among the data components.
	A common, truly multivariate approach consists in computing the 
	log-likelihood of the datastream and compare the distribution of the 
	log-likelihood over different time windows 
	(Section~\ref{sec:ProblemFormulation}). In practice, computing the 
	log-likelihood is an effective way to reduce the multivariate change-detection 
	problem to a univariate one, thus easily addressable by any scalar CDT. 
	Several CDTs for multivariate datastreams pursue this approach, and compute the log-likelihood with respect to a model fitted to a training set of stationary data:~\cite{Kuncheva2013_TKDE} uses Gaussian mixtures,~\cite{Krempl2011_APT,Polikar2014_COMPOSE} use nonparametric density 
	models.
	Other CDTs have been designed upon specific multivariate 
	statistics~\cite{schilling1986multivariate,agarwal2005empirical,lung2011robust,wang2002mean,Ditzler2011_Hellinger,nguyen2014constrained}.
	In the classification literature, where changes in the distribution are 
	referred to as concept-drift~\cite{Zliobaite2014}, changes are typically detected by monitoring the scalar 
	sequence of classification errors over 
	time~\cite{Gama2004,TNNLS13,bifet2007learning,Ross2012}.
	
	Even though this problem is of utmost relevance in datastream mining, no theoretical or experimental 
	study investigate how the data dimension $d$ impacts on the 
	change detectability. 
	In Section~\ref{sec:Theory}, we consider change-detection problems in $\mathbb{R}^d$ and 
	investigate how $d$ affects the detectability of a change when monitoring the 
	log-likelihood of the datastream. In this respect, we show that the symmetric
	Kullback-Leibler divergence (sKL) between pre-change and post-change distributions is an appropriate measure of the \emph{change magnitude}, and we introduce the \emph{Signal-to-Noise Ratio} \emph{of 
		the change} (SNR) to quantitatively assess the change detectability when monitoring 
	the log-likelihood. 
	
	Then, we show that the detectability of changes having a given magnitude progressively reduces when $d$ increases. 
	We refer to this phenomenon as \emph{detectability loss}, and we analytically demonstrate that, 
	in case of Gaussian random variables, the change detectability is upperbounded by a function that decays as $1 / d$. We demonstrate that detectability loss occurs also in non Gaussian cases as far as data components are independent, and we show that it affects also real-world datasets, which we approximate by Gaussian mixtures in our empirical analysis (Section~\ref{sec:Experiments}). Most importantly, detectability loss is not a consequence of density-estimation problems, as it holds either when data distribution is estimated from training samples or known. Our results indicate that detectability loss is a potentially harmful also at reasonably low-dimensions (e.g., 10) and not only in Big-Data scenarios.
	
	\section{Monitoring the Log-Likelihood}\label{sec:ProblemFormulation}
	
	\subsection{The Change Model} We assume that, in stationary conditions, the 
	datastream $\{\x(t), t = 1,\dots \}$ contains independent and identically 
	distributed (i.i.d.) random vectors $\x(t)\in \mathbb{R}^d$, drawn from a 
	random variable $\mX$ having probability-density-function (pdf) $\phi_0$, that 
	for simplicity we assume continuous, strictly positive and bounded. Here, $t$ denotes the time 
	instant, bold letters indicate column vectors, and $'$ is the matrix transpose 
	operator.
	
	For the sake of simplicity, we consider permanent changes $\change$ affecting 
	the expectation and/or correlation of $\mathcal X$:
	\begin{equation}\label{eq:change_model}
		\x(t) \sim 
		\begin{cases}
			\phi_0 \;\;\;\; t < \tau \\
			\phi_1 \;\;\;\; t \geq \tau  
		\end{cases},  \text{ where } \phi_1 (\mb x) = \phi_0(Q\mb x + \mb v)\,,
	\end{equation}
	where $\tau$ is the unknown change point, $\mb v \in \mathbb{R}^d$ changes the location $\phi_0$, and $Q\in O(d)\subset \mathbb{R}^{d \times d}$ is an orthogonal matrix that modifies the correlation among the components of $\mb x$. 
	This rather general change-model requires a truly multivariate monitoring scheme: changes affecting only the correlation among components of $\mb x$ cannot be perceived by analyzing each component individually, or by extracting straightforward features 
	(such as the norm) out of vectors $\x(t)$\footnote{We do not consider changes affecting data dispersion as these can be detected by monitoring the Euclidean norm of $\x(t)$.}. 
	
	\subsection{The Considered Change-Detection Approach}\label{subsec:CD_Approach}
	We consider the popular change-detection approach that consists in monitoring the 
	log-likelihood of $\x(t)$ with respect to 
	$\phi_0$~\cite{Kuncheva2013_TKDE,Song2007,sullivan2000change}: 
	\begin{equation}\label{eq:LogLikelihoodGeneral}
		\l(\x(t)) = \log(\phi_0(\x(t)))\,, \ \forall t\,.
	\end{equation}
	
	We denote by $L =\{\l(\x(t)), t = 1,\dots,\}$ the sequence of log-likelihood 
	values, and observe that in stationary conditions, $L$ contains i.i.d. data drawn from a scalar 
	random variable. When $\mX$ undergoes a change, the distribution of 
	$\l(\cdot)$ is also expected to change. Thus, changes $\change$ can be detected by comparing 
	the distribution of $\l(\cdot)$ over $W_P$ and $W_R$, two non-overlapping 
	windows of $L$, where $W_P$ refers to past data (that we assume are 
	generated from $\phi_0$), and $W_R$ refers to most recent ones (that are 
	possibly generated from $\phi_1$). In practice, a suitable test statistic 
	$\mathcal{T}\left(W_P, W_R\right)$, such as the t-statistic, Kolmogorov-Smirnov or Lepage 
	\cite{Lepage1974}, is computed to compare $W_P$ and $W_R$. In an hypothesis testing framework, this corresponds to formulating a test having as null hypothesis ``\emph{samples in $W_P$ and $W_R$ are from the same distribution}''.
	When $\mathcal{T}\left(W_P, W_R\right) > h$ we can safely 
	consider that the log-likelihood values over $W_P$ and $W_R$ are 
	from two different distributions, indicating indeed a change in $\mX$. The threshold $h>0$ controls the test significance. 
	
	There are two important aspects to be considered about this change-detection approach. First, that comparing data on different windows is not a genuine sequential monitoring scheme. However, this mechanism is at the core of several online change-detection methods~\cite{Kuncheva2013_TKDE,Song2007,bifet2007learning,RossTechnometrics2011}. Moreover, the power of the test $\mathcal{T}\left(W_P, W_R\right) > h$, namely the probability of rejecting the null hypothesis when the alternative holds, indicates the effectiveness of the test statistic $\mathcal{T}$ when the same is used in sequential-monitoring techniques. 
	Second, that $\phi_0$ in \eqref{eq:LogLikelihoodGeneral} is often unknown and has to be preliminarily estimated from a training set of stationary data. Then, $\phi_0$ is simply replaced by its estimate $\widehat{\phi}_0$. In practice, it is fairly reasonable to assume a training set of stationary data is given, while it is often unrealistic to assume $\phi_1$ is known, since the datastream might change unpredictably.
	
%
	
	\section{Theoretical Analysis}\label{sec:Theory}
The section sheds light on the relationship between change detectability and 
$d$. To this purpose, we introduce: \emph{i}) a measure of the \emph{change 
	magnitude}, and \emph{ii}) an indicator that quantitatively assesses \emph{change detectability}, 
namely how difficult is to detect a change when monitoring $\l(\cdot)$ as described in Section \ref{subsec:CD_Approach}.
Afterward, we can study the influence of $d$ on the change detectability provided that changes $\change$ have a constant magnitude. 

\subsection{Change Magnitude}
The magnitude of $\change$ can be naturally measured by the symmetric Kullback-Leibler divergence between $\phi_0$ and $\phi_1$ 
(also known as Jeffreys divergence):
\begin{equation}\label{eq:KL_general}
	\begin{aligned}
		\textnormal{sKL}&(\phi_0, \phi_1) := \textnormal{KL}(\phi_0, \phi_1) + \textnormal{KL}(\phi_1, \phi_0) \\ 
		&= \int_{\mathbb{R}^{d}} \log\frac{\phi_0(\mb x)}{\phi_1(\mb 
			x)}\phi_0(\mb x) d\mb x  + \int_{\mathbb{R}^{d}} \log\frac{\phi_1(\mb x)}{\phi_0(\mb 
			x)}\phi_1(\mb x) d\mb x \,.
	\end{aligned}
\end{equation}
This choice is supported by the Stein's Lemma~\cite{cover2012elements}, which states that $\textnormal{KL}(\phi_0, \phi_1)$ yields an upper-bound for the power of parametric hypothesis tests that determine whether a given sample population is generated from $\phi_0$ (null hypothesis) or $\phi_1$ (alternative hypothesis). In practice, large values of $\KLPhi$ indicate changes that are very apparent, since hypothesis tests designed to detect either $\change$ or $\phi_1 \to \phi_0$ can be very powerful.


	
\subsection{Change Detectability}\label{subsec:changeDetectability}
We define the following indicator to quantitatively 
assess the detectability of a change when monitoring $\l(\cdot)$.
\begin{definition} The \textnormal{signal-to-noise ratio} (SNR) of the change $\change$ is defined as:
	\begin{equation}\label{eq:SNRofTheChange}
		\SNRPhi := \frac{\left(\ezero [\l(\x)] - \eone[\l(\x)]\right)^2}{\varzero[\l(\x)] + \varone[\l(\x)]},
	\end{equation}
	where $\textnormal{var}[\cdot]$ denotes the variance of a random variable. 
\end{definition}
In particular, $\SNRPhi$ measures the extent to which $\change$ is detectable by monitoring the expectation of $\l(\cdot)$. In fact, the numerator of \eqref{eq:SNRofTheChange} corresponds to the shift introduced by $\change$ in the expectation of $\l(\cdot)$ (i.e.,  the relevant information, the \emph{signal}) which is easy/difficult to detect relatively to its random fluctuations (i.e., the \emph{noise}), which are assessed in the denominator of \eqref{eq:SNRofTheChange}. 
Note that, if we replace the expectations and the variances in $\eqref{eq:SNRofTheChange}$ by their sample estimators, we obtain that $\SNRPhi$ corresponds -- up to a scaling factor -- to the square statistic of a Welch's $t$-test \cite{welch1947generalization}, that detects changes in the expectation of two sample populations. This is another argument supporting the use of $\SNRPhi$ as a measure of change detectability.

The following proposition relates the change magnitude $\KLPhi$ with the numerator of \eqref{eq:SNRofTheChange}.
\begin{proposition}\label{prop:KL_change}
	Let us consider a change $\change$ such that
	\begin{equation}\label{eq:change_model_prop1}
		\phi_1 (\mb x) = \phi_0(Q\mb x + \mb v)\,
	\end{equation}
	where $Q \in \mR^{d\times d}$ is orthogonal and $\mb{v} \in \mR^{d}$. Then, it holds:
	\begin{equation}\label{eq:Prop1_Statement}
		\KLPhi \geq \underset{\x \sim \phi_0}{E} [\l(\x)] - \underset{\x \sim \phi_1}{E} 
		[\l(\mb x)]
	\end{equation}
\end{proposition}
\begin{proof}
	From the definition of $\KLPhi$ in \eqref{eq:KL_general} it follows
	\begin{equation}\nonumber 
		\begin{aligned}
			\KLPhi &= \underset{\x \sim \phi_0}{E} [\LGZERO] -  
			\underset{\x \sim \phi_0}{E} [\LGUNO] + \\
			&+ \underset{\x \sim \phi_1}{E} [\LGUNO] -  
			\underset{\x \sim \phi_1}{E} [\LGZERO]\,.
		\end{aligned}
	\end{equation}
	Since $\l(\cdot) = \log\left(\phi_0(\cdot)\right)$, \eqref{eq:Prop1_Statement} holds if and only if
	\begin{equation}\label{eq:Prop1_diff_exp}
		\underset{\x \sim \phi_1}{E} [\LGUNO] - \underset{\x \sim \phi_0}{E} [\LGUNO] \geq 0. 
	\end{equation}
	From \eqref{eq:change_model_prop1} it follows that $\phi_0(\x) = \phi_1(Q'(\mb x -\mb v))$, thus, by replacing the mathematical expectations with their integral expressions, \eqref{eq:Prop1_diff_exp} becomes
	\begin{equation}\label{eq:Prop1_eq2}
		\int \LGUNO \phi_1(\mb x) d\mb x - \int \log \left( \phi_1(\x)\right) \phi_1(Q'(\mb x -\mb v)) d\mb x \geq 0\\
	\end{equation}
	Let us define $\mb y = Q'(\mb x - \mb v)$, then $\mb x = Q\mb y + \mb v$ and $d\x=|\textnormal{det}(Q)|d\mb y = d\mb y$, since $Q$ is orthogonal. Using this change of variables in the second summand of \eqref{eq:Prop1_eq2} we obtain 
	\begin{equation}\label{eq:last_eq_prop1}
		\int \LGUNO \phi_1(\mb x) d\mb x - \int \log \left( \phi_1(Q\mb y + \mb v)\right) \phi_1(\mb y) d\mb y \geq 0.
	\end{equation}
	Finally, defining $\phi_2(\mb y) := \phi_1(Q \mb y + \mb v)$ turns \eqref{eq:last_eq_prop1} into
	\begin{equation}\label{eq:Prop1_final}
		\int \LGUNO \phi_1(\mb x) d\mb x - \int \log \left(\phi_2(\mb y)\right) \phi_1(\mb y)) d\mb y \geq 0,
	\end{equation}
	which holds since the left-hand-side of \eqref{eq:Prop1_final} is $\textnormal{KL}(\phi_1, \phi_2)$.\\
\end{proof}

\subsection{Detectability Loss}
It is now possible to investigate the intrinsic challenge of change-detection problems when data dimension increases. In particular, we study how the change detectability (i.e., $\SNRPhi$) varies when $d$ 
increases and changes $\change$ preserve constant magnitude (i.e., $\KLPhi = const$). 
Unfortunately, since there are no general expressions for the variance of 
$\l(\cdot)$, we have to assume a specific distribution for $\phi_0$ to carry out 
any analytical development. As a relevant example, we consider Gaussian 
random variables, which enable a simple expression of $\l(\cdot)$. The following 
theorem demonstrates the \emph{detectability loss} for Gaussian distributions, namely that 
$\SNRPhi$ decays as $d$ increases.
\begin{theorem}\label{theorem:gaussian_detectability}
	Let $\phi_0 = \mathcal{N}(\mu_0, \Sigma_0)$ be a $d$-dimensional Gaussian pdf and $\phi_1 = \phi_0(Q\mb x +\mb v)$, where $Q \in \mR^{d\times d}$ is orthogonal and $\mb{v} \in \mR^{d}$. Then, it holds
	\begin{equation}\label{eq:detectability_loss}
		\SNRPhi \leq \frac{C}{d}
	\end{equation}
	where the constant $C$ depends only on $\KLPhi$.
\end{theorem}
\begin{proof}
	Basic algebra leads to the following expression for $\l(\x)$ when $\phi_0 = \mathcal{N}(\mu_0, \Sigma_0)$:
	\begin{equation}\label{eq:LogLikelihoodGaussian}
		\l(\x) = - \frac{1}{2}\log\left((2\pi)^d\textnormal{det}(\Sigma_0)\right) 
		-\frac{1}{2}(\x-\m_0)'\Sigma_0^{-1}(\x-\m_0)\,.
	\end{equation}
	The first term in the right-hand-side of  \eqref{eq:LogLikelihoodGaussian} is constant, while the second term is distributed as a chi-squared having $d$ degrees of freedom. 
	Therefore,
	\begin{equation}\label{eq:varince_loglikelihood}
		\underset{\x \sim \phi_0}{\text{var}}[\l(\x)] = 
		\text{var}\left[-\frac{1}{2} \chi^2(d)\right] = \frac{d}{2}\,.
	\end{equation}
	Then, from the definition of $\SNRPhi$ in \eqref{eq:SNRofTheChange} and Proposition \ref{prop:KL_change}, it follows that
	\begin{equation}\label{eq:SNRofTheChange_Gaussian}\nonumber
		\begin{aligned}
			\SNRPhi &\leq \frac{\KLPhi^2}{\underset{\x \sim \phi_0}{\text{var}[\l(\x)]} + 
				\underset{\x \sim \phi_1}{\text{var}[\l(\x)]}} \leq 
			\frac{\KLPhi^2}{\underset{\x \sim \phi_0}{\text{var}[\l(\x)]}} \\
			&=\frac{\KLPhi^2}{d/2}  = \frac{C}{d}\,.
		\end{aligned}
	\end{equation}
\end{proof}
Theorem \ref{theorem:gaussian_detectability} shows detectability loss for Gaussian distributions. In fact, when $d$ increases and $\KLPhi$ remains constant, $\SNRPhi$ is upper-bounded by a function that monotonically decays as $1/d$. The decaying trend of $\SNRPhi$ indicates that detecting changes becomes more difficult when $d$ increases. Moreover, the decaying rate does not depend on $\KLPhi$, thus this problem equally affects all possible changes $\change$ defined as in 	\eqref{eq:change_model}, disregarding their magnitude.

\subsection{Discussion}\label{subsec:Discussion}
First of all, let us remark that Theorem \ref{theorem:gaussian_detectability} implicates detectability loss only when $\KLPhi$ is kept constant. Assuming constant change magnitude is necessary to correctly investigate the influence of the sole data dimension $d$ on the change detectability. In fact, when the change magnitude increases with $d$, changes might become even easier to detect as $d$ grows. This is what experiments in~\cite{Zimek2012}(Section 2.1) show, where outliers\footnote{Even though similar techniques can be sometimes used for both change-detection and anomaly-detection, the two problems are intrinsically different, since the former aims at recognizing process changes, while the latter at identifying spurious data.} become easier to detect when $d$ increases. However, in that experiment, the change-detection problem becomes easier as $d$ increases, since each component of $\x$ carries additional information about the change, thus increases $\KLPhi$.

Detectability loss can be also proved when $\phi_0$ is non Gaussian, as far as its components are independent. In fact, if $\phi_0(\mb x) = \prod_{i = 0}^d\phi^{(i)}_0(x^{(i)})$, where~$(\cdot)^{(i)}$ denotes either the marginal of a pdf or the component of a vector, it follows 
\begin{equation}\label{eq:varIndependent}
	\underset{\x \sim \phi_0}{\textnormal{var}[\l(\x)]} = \sum_{i = 0}^d \underset{\x \sim \phi_0}{\textnormal{var}}\left[\log\left(\phi^{(i)}_0(x^{(i)})\right)\right]\,,
\end{equation}
since $\log(\phi_0^{(i)}(x^{(i)}))$ are independent. Clearly, \eqref{eq:varIndependent} increases with $d$, since its summands are positive. Thus, also in this case, the upperbound of $\SNRPhi$ decays with $d$ when $\KLPhi$ is kept constant.

Remarkably, detectability loss does not depend on how the change $\change$ affects $\mathcal{X}$. 
Our results hold, for instance, when either $\change$ affects all the components of $\mathcal{X}$ or some of them remain irrelevant for change-detection purposes. Moreover, detectability loss occurs independently of the specific change-detection method used on the log-likelihood (e.g. sequential analysis, or window comparison), as our results concern $\SNRPhi$ only.


In the next section we show that detectability loss affects also real-world change-detection problems. To this purpose, we design a rigorous empirical analysis to show that the power of customary hypothesis tests actually decreases with $d$ when data are non Gaussian and possibly dependent. 



	
\section{Empirical Analysis}\label{sec:Experiments}
Our empirical analysis has been designed to address the following goals:
\emph{i}) showing that $\SNRPhi$, which is the underpinning element of our theoretical result, is a suitable measure of change detectability. In particular, we show that the power of hypothesis tests able to detect both changes in mean and in variance of $\l(\cdot)$ also decays.
\emph{ii}) Showing that detectability loss is not due to density-estimation problems, but it becomes a more serious issue when $\phi_0$ is estimated from training data.
\emph{iii}) Showing that detectability loss occurs also in Gaussian mixtures, and 
\emph{iv}) showing that detectability loss occurs also on high-dimensional real-world datasets, which are far from being Gaussian or having independent components. We address the first two points in Section \ref{subsec:GaussianDatastreams}, while the third and fourth ones in Sections \ref{subsec:GaussianMixtureModels} and \ref{subsec:RealWorld}, respectively. 

In our experiments, the change-detection performance is assessed by numerically computing the power of two customary hypothesis tests, namely the Lepage \cite{Lepage1974} and the one-sided $t$-test\footnote{We can  assume that $\change$ decreases the expectation of $\l$ since  $\underset{\x \sim \phi_0}{E} [\LGZERO] - \underset{\x \sim \phi_1}{E} [\LGZERO] \geq 0$ follows from \eqref{eq:Prop1_diff_exp}.} on data windows $W_P$ and $W_R$ which contains 500 data each. As we discussed in Section \ref{subsec:changeDetectability}, the t-statistic on the log-likelihood is closely related to $\SNRPhi$, while the Lepage is a nonparametric statistic that detects both location and scale changes\footnote{The Lepage statistic is defined as the sum of the squares of the Mann-Whitney and Mood statistics, see also~\cite{RossTechnometrics2011}.}. To compute the power, we set $h$ to guarantee a significance level\footnote{The value of $h$ for the Lepage test is given by the asymptotic approximation of the statistic in \cite{Lepage1974}.} $\alpha=0.05$. Following the procedure in Appendix, we synthetically introduce changes $\change$ having $\KLPhi = 1$ which, in the univariate Gaussian case, corresponds to $\mb v$ equals to the standard deviation of $\phi_0$.

\subsection{Gaussian Datastreams}\label{subsec:GaussianDatastreams}

We generate Gaussian datastreams having dimension $d\in\{1,2,4,8,16,32,64,128\}$ and, for each value 
of $d$, we prepare 10000 runs, with $\phi_0 = 
\mathcal{N}(\mu_0,\Sigma_0)$ and $\phi_1 = \mathcal{N}(\mu_1,\Sigma_1)$. The 
parameters $\mu_0\in \mathbb{R}^d$ and $\Sigma_0\in \mathbb{R}^{d \times d}$ 
have been randomly generated, while $\mu_1\in \mathbb{R}^d$ and $\Sigma_1\in 
\mathbb{R}^{d \times d}$ have been set to yield $\KLPhi = 1$ (see Appendix). In each run we generate 1000 samples: $\{\mb x(t),\, t = 1,\dots, 500\}$ from $\phi_0$, and $\{\mb x(t),\, t = 501,\dots, 1000\}$ from $\phi_1$. Then, we compute the datastream $L = \{\l(\x(t)),\, t = 1,\dots, 1000\}$, and define $W_P = \{\l(\x(t)), t = 1,\dots, 500\}$ and $W_R = \{\l(\x(t)), t = 501,\dots, 1000\}$.

We repeat the same experiment replacing $\phi_0$ with its estimate 
$\widehat{\phi}_0(\x)$, where $\widehat{\mu}_0$ and $\widehat{\Sigma}_0$ are 
computed using the sample estimators over an additional training set $TR$ whose size grows 
linearly with $d$, i.e. $\#TR = 100\cdot d$. We denote by $\widehat{L} = 
\{\widehat{\l}(\x(t)), t = 1,\dots, 1000\}$ the sequence of estimated 
log-likelihood values. Finally, we repeat the whole experiments keeping $\#TR = 100$ for any value of $d$, and we denote by $\widehat{L}_{100}$ the corresponding sequence of log-likelihood 
values.

Figure~\ref{fig:results}(a) shows that the power of both the Lepage and one-sided $t$-test substantially decrease when $d$ increases. This result is coherent with our theoretical analysis of Section \ref{sec:Theory}, and confirms that $\SNRPhi$ is a suitable measure of change detectability. While it is not surprising that the power of the $t$-test decays, given its connection with the $\SNRPhi$, it is remarkable that the power of the Lepage test also decays, as this fact indicates that it becomes more difficult to detect both changes in the mean and in the dispersion of $L$. The decaying power of both tests indicates that the corresponding test statistics decrease with $d$, which imply larger detection delays when using this statistics in sequential monitoring schemes.

Note that detectability loss is not due to density-estimation issues, but rather to the fact that the change-detection problem becomes intrinsically more challenging, as it occurs in the ideal case where $\phi_0$ is known (solid lines). When $\l$ is computed from an estimated $\widehat{\phi}_0$ (dashed and dotted lines), the problem becomes even more severe, and worsens when the number of training data does not grow with $d$ (dotted lines).

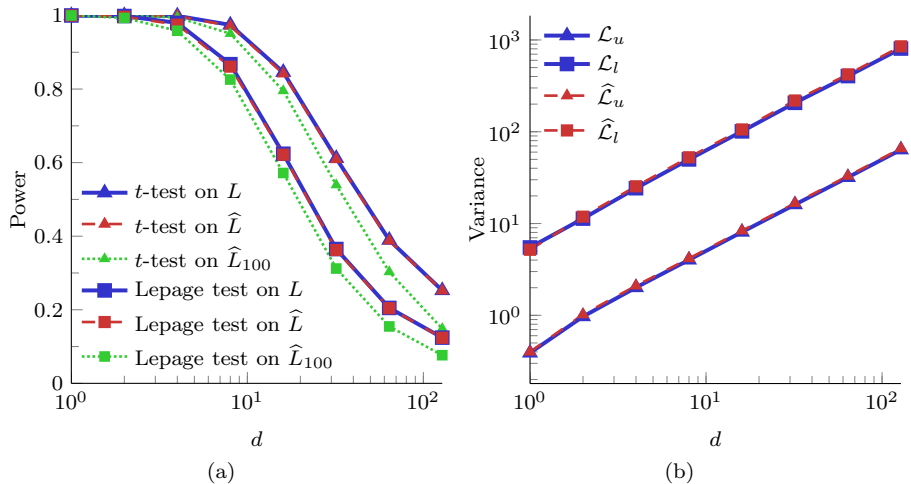
\begin{figure*}[t!]
	\centering
	\subfigure[]{\label{fig:power_Gaussian}\begin{tikzpicture}

\begin{axis}[%
width=0.4\columnwidth,
height=0.4\columnwidth,
	scale only axis,
xticklabel style = {font = \small},
xmin=1,
xmax=128,
xmode = log,
xlabel={$d$},
ymin=0,
ymax=1,
ylabel={Power},
ylabel style = {anchor = near ticklabel, at = {(0.15,0.5)}},
ytick = {0,0.2,0.4,0.6,0.8,1},
minor ytick={0.1,0.3,0.5,0.7,0.9},
yticklabel style = {font = \small},
axis x line*=bottom,
axis y line*=left,
legend style={at={(0.0,0.01)},
	anchor=south west, 
	draw = none,
	fill = none,
	font=\small,
	legend cell align=left},
]

\addplot [
color=myblue,
solid,
line width=1.5pt,
mark=triangle*,
mark size = 2pt,
mark options={solid},
]
table{gauss_t-test_L.txt};
\addlegendentry{$t$-test on $L$}

\addplot [
color=myred,
dash pattern = on 5pt off 5pt,
line width=1pt,
mark=triangle*,
mark size = 1.8pt,
mark options={solid},
]
table{gauss_t-test_L_hat.txt};
\addlegendentry{$t$-test on $\widehat L$}

\addplot [
color=mygreen,
dash pattern = on 1pt off 1pt,
line width=1pt,
mark=triangle*,
mark size = 1.5pt,
mark options={solid},
]
table{gauss_t-test_L_hat_100.txt};
\addlegendentry{$t$-test on $\widehat L_{100}$}

\addplot [
color=myblue,
solid,
line width=1.5pt,
mark=square*,
mark size = 2pt,
mark options={solid},
]
table{gauss_LP_test_L.txt};
\addlegendentry{Lepage test on $L$}

\addplot [
color=myred,
dash pattern = on 5pt off 5pt,
line width=1pt,
mark=square*,
mark size = 1.8pt,
mark options={solid},
]
table{gauss_LP_test_L_hat.txt};
\addlegendentry{Lepage test on $\widehat L$}

\addplot [
color=mygreen,
dash pattern = on 1pt off 1pt,
line width=1pt,
mark=square*,
mark size = 1.5pt,
mark options={solid},
]
table{gauss_LP_test_L_hat_100.txt};
\addlegendentry{Lepage test on $\widehat L_{100}$}

\end{axis}
\end{tikzpicture}%
}\;
	\subfigure[]{\label{fig:variance_all}\begin{tikzpicture}

\begin{axis}[%
width=0.4\columnwidth,
height=0.4\columnwidth,
	scale only axis,
xticklabel style = {font = \small},
xmin=1,
xmax=128,
xlabel={$d$},
xmode = log,
ymin=0,
ymode = log,
ylabel={Variance},
ylabel style = {anchor = near ticklabel, at = {(0.15,0.5)}},
yticklabel style = {font = \small},
axis x line*=bottom,
axis y line*=left,
legend style={at={(0.01,1)},
	anchor=north west, 
	draw = none,
	fill = none,
	font=\small,
	legend cell align=left},
]

\addplot [
color=myblue,
solid,
line width=1.5pt,
mark=triangle*,
mark size = 2pt,
mark options={solid},
]
table{gmm_variance_L_upper.txt};
\addlegendentry{$\l_u$}

\addplot [
color=myblue,
solid,
line width=1.5pt,
mark=square*,
mark size = 2pt,
mark options={solid},
]
table{gmm_variance_L_lower.txt};
\addlegendentry{$\l_l$}

\addplot [
color=myred,
dash pattern = on 5pt off 5pt,
line width=1pt,
mark=triangle*,
mark size = 1.8pt,
mark options={solid},
]
table{gmm_variance_L_hat_upper.txt};
\addlegendentry{$\widehat{\l}_u$}
%
\addplot [
color=myred,
dash pattern = on 5pt off 5pt,
line width=1pt,
mark=square*,
mark size = 1.8pt,
mark options={solid},
]
table{gmm_variance_L_hat_lower.txt};
\addlegendentry{$\widehat{\l}_l$}

\end{axis}
\end{tikzpicture}%
}\\
	\caption{(a) Power of the Lepage and one-sided $t$-test empirically computed on sequences generated as in Section~\ref{subsec:GaussianDatastreams}. Detectability loss clearly emerges when the log-likelihood is computed 
	using $\phi_0$ (denoted by $L$) or its estimates fitted on $100\cdot d$ samples ($\widehat{L}$) or from 100 samples ($\widehat{L}_{100}$). 
	(b) The sample variance of $\mathcal L_u(\cdot)$ \eqref{eq:likelihoodGM_upper} and $\mathcal L_l(\cdot)$ \eqref{eq:likelihoodGM_lower} computed as in Section \ref{subsec:GaussianMixtureModels}. As in the Gaussian case, both these 
	variances grow linearly with $d$ and similar results hold when using 
	$\widehat\phi_0$, which is estimate from $200\cdot d$ training data. 
	}
	\label{fig:results}
\end{figure*}
	
	Figure~\ref{fig:results}(a) shows that the power of both the Lepage and $t$-test substantially decrease when $d$ increases. This result is coherent with our theoretical analysis of Section \ref{sec:Theory}, and confirms that $\SNRPhi$ is a suitable measure of change detectability. While it is not surprising that the power of the $t$-test decays considered its connection with the $\SNRPhi$, it is remarkable that the power of the Lepage test also decays, indicating that changes become more difficult to be detected both in the mean and in the dispersion of $L$. 
	
	We interpret this result as a consequence of the intrinsic worsening of change detectability when $d$ increases. Moreover, detectability loss is not due to density-estimation issues, as it affects also ideal case where $\phi_0$ is known (solid lines). In the more realistic cases where $\l$ is computed from an estimated $\widehat{\phi}_0$ (dashed and dotted lines), the problem become more severe, in particular when the number of training data does not increase with $d$ (dotted lines).

	\subsection{Gaussian mixtures}\label{subsec:GaussianMixtureModels}
	
We now consider $\phi_0$ and $\phi_1$ as Gaussian mixtures, to prove that 
detectability loss occurs also when datastreams are generated/approximated by more general distribution models.
Mimicking the proof of Theorem \ref{theorem:gaussian_detectability}, we show that when $d$ increases and $\KLPhi$ is kept constant, the upper-bound of $\SNRPhi$ decreases. To this purpose, it is enough to show that $\varzero[\l(\x)]$ increases with $d$.

The pdf of a mixture of $k$ Gaussians is
\begin{equation}\label{eq:GaussianMixture}
	\begin{aligned}
		\phi_0(\x) &= \sum_{i = 1}^{k} \lambda_{0,i}\mathcal 
		N(\m_{0,i},\Sigma_{0,i})(\x)= \\
		&=\sum_{i = 1}^{k}
		\frac{\lambda_{0,i}}{(2\pi)^{d/2}\textnormal{det}(\Sigma_{0,i})^{1/2}} 
		e^{-\frac{1}{2}(\x-\m_{0,i})'\Sigma_{0,i}^{-1}(\x-\m_{0,i})},
	\end{aligned}
\end{equation}
where $\lambda_{0,i} > 0$ is the weight of the $i$-th Gaussian 
$\mathcal{N}\left(\m_{0,i},\Sigma_{0,i}\right)$.
Unfortunately, the log-likelihood $\l(\x)$ of a Gaussian mixture does not admit an expression similar to~\eqref{eq:LogLikelihoodGaussian} and two approximations are typically used to avoid severe numerical issues when 
$d \gg 1$.

%
%
The first approximation consists in considering only the Gaussian of the mixture yielding the largest likelihood, as in~\cite{Kuncheva2013_TKDE} 
i.e.,
\begin{equation}\label{eq:likelihoodGM_upper}
	\begin{aligned}
		\l_u(\x) &= -\frac{k\lambda_{0,i^*}}{2}\Bigl(\log\left((2\pi)^d\textnormal{det}(\Sigma_{0,i^*})\right) + \\
		& \qquad + (\x-\m_{0,i^*})'\Sigma_{0,i^*}^{-1}(\x-\m_{0,i^*}) \Bigr)
	\end{aligned}
\end{equation}
where $i^*$ is defined as
\begin{equation}\label{eq:maximumLikelihoodGaussian}\nonumber
	i^* = \underset{i = 1, \dots,k}{\text{argmax}} \left(  
	\frac{\lambda_{0,i}}{(2\pi)^{d/2}\textnormal{det}(\Sigma_{0,i})^{1/2}} 
	e^{-\frac{1}{2}(\x-\m_{0,i})'\Sigma_{0,i}^{-1}(\x-\m_{0,i})}  \right). 
\end{equation}
The second approximation we consider is:
\begin{equation}\label{eq:likelihoodGM_lower}
	\begin{aligned}
		\l_l(\x) &= -\frac{1}{2}\sum_{i=1}^k \lambda_{0,i}\Bigl(\log\left((2\pi)^d\textnormal{det}(\Sigma_{0,i})\right)+\\
		&\qquad+(\x-\m_{0,i})'\Sigma_{0,i}^{-1}(\x-\m_{0,i})\Bigr),
	\end{aligned}
\end{equation}
that is a lower bound of $\l(\cdot)$ due 	to the Jensen inequality.
	
We consider the same values of $d$ as in Section~\ref{subsec:GaussianDatastreams} and, for each of these, we generate 1000 datastreams each containing 500 data drawn from the Gaussian 
mixture $\phi_0$. We assume $k=2$ and $\lambda_{0,1}=\lambda_{0,2}=0.5$, while the paramters $\mu_{0,1}$, $\mu_{0,2}$, $\Sigma_{0,1}$, $\Sigma_{0,2}$ are randomly generated. We then compute the sample variance of 
both $\l_u$ and $\l_l$ over each datastream and report their average in 
Figure~\ref{fig:results}(b). As in Section \ref{subsec:GaussianDatastreams}, 
we repeat this experiment estimating $\widehat{\phi}_0$ from a training set 
containing $200 \cdot d$ additional samples, then computing $\widehat{\l}_u$ and $\widehat{\l}_l$.

		\begin{figure*}[t!]
			\centering
			\subfigure[]{\label{fig:wine_stat}\begin{tikzpicture}

\begin{axis}[%
width=0.4\columnwidth,
height=0.4\columnwidth,
scale only axis,
xtick = {10,20,30,40,50},
xticklabel style = {font = \small},
xmin=1,
xmax=50,
xlabel={$d$},
ymin=0,
ymax=1,
ylabel={Power},
ylabel style = {anchor = near ticklabel, at = {(0.15,0.5)}},
ytick = {0,0.2,0.4,0.6,0.8,1},
minor ytick={0.1,0.3,0.5,0.7,0.9},
yticklabel style = {font = \small},
axis x line*=bottom,
axis y line*=left,
legend style={at={(1,1)},
	anchor=north east, 
	draw = none,
	fill = none,
	font=\small,
	legend cell align=left},
]

\addplot [
color=myred,
solid,
line width=1.5pt,
mark=triangle*,
mark size = 2pt,
mark options={solid},
mark repeat = {5},
]
table{particle_t-test_L_hat_upper.txt};
\addlegendentry{$t$-test on $\widehat{\l}_u$}

\addplot [
color=myred,
dash pattern = on 5pt off 5pt,
line width=1pt,
mark=triangle*,
mark size = 1.8pt,
mark options={solid},
mark repeat = {5},
]
table{particle_t-test_L_hat_lower.txt};
\addlegendentry{$t$-test on $\widehat{\l}_l$}

\addplot [
color=myblue,
solid,
line width=1.5pt,
mark=square*,
mark size = 2pt,
mark options={solid},
mark repeat = {5},
]
table{particle_LP-test_L_hat_upper.txt};
\addlegendentry{Lepage test on $\widehat{\l}_u$}

\addplot [
color=myblue,
dash pattern = on 5pt off 5pt,
line width=1pt,
mark=square*,
mark size = 1.8pt,
mark options={solid},
mark repeat = {5},
]
table{particle_LP-test_L_hat_lower.txt};
\addlegendentry{Lepage test on $\widehat{\l}_l$}

\end{axis}
\end{tikzpicture}%
}\;	
			\subfigure[]{\label{fig:powers_wine}\begin{tikzpicture}

\begin{axis}[%
width=0.4\columnwidth,
height=0.4\columnwidth,
	scale only axis,
xticklabel style = {font = \small},
xmin=1,
xmax=11,
xlabel={$d$},
ymin=0,
ymax=1,
ylabel={Power},
ylabel style = {anchor = near ticklabel, at = {(0.15,0.5)}},
ytick = {0,0.2,0.4,0.6,0.8,1},
minor ytick={0.1,0.3,0.5,0.7,0.9},
yticklabel style = {font = \small},
axis x line*=bottom,
axis y line*=left,
]

\addplot [
color=myred,
solid,
line width=1.5pt,
mark=triangle*,
mark size = 2pt,
mark options={solid},
]
table{wine_t-test_L_hat_upper.txt};

\addplot [
color=myred,
dash pattern = on 5pt off 5pt,
line width=1pt,
mark=triangle*,
mark size = 1.8pt,
mark options={solid},
]
table{wine_t-test_L_hat_lower.txt};

\addplot [
color=myblue,
solid,
line width=1.5pt,
mark=square*,
mark size = 2pt,
mark options={solid},
]
table{wine_LP-test_L_hat_upper.txt};

\addplot [
color=myblue,
dash pattern = on 5pt off 5pt,
line width=1pt,
mark=square*,
mark size = 1.8pt,
mark options={solid},
]
table{wine_LP-test_L_hat_lower.txt};

\end{axis}
\end{tikzpicture}%
}
			\caption{\small Detectability loss on the Particle Dataset (a) and Wine Dataset (b), approximated by a mixture of 2 and 4 Gaussians, respectively, using both  $\widehat{\l}_u$ \eqref{eq:likelihoodGM_upper} and $\widehat{\l}_l$ \eqref{eq:likelihoodGM_lower}. The powers of both Lepage and $t$-test decay, confirming the detectability loss. 
			 The tests based on $\widehat{\l}_u$ outperform the corresponding ones based on $\widehat{\l}_l$ because this latter approximation yields a larger variance, as can be seen in Fig \ref{fig:results}(b).} 
			\label{fig:results_real}
		\end{figure*}
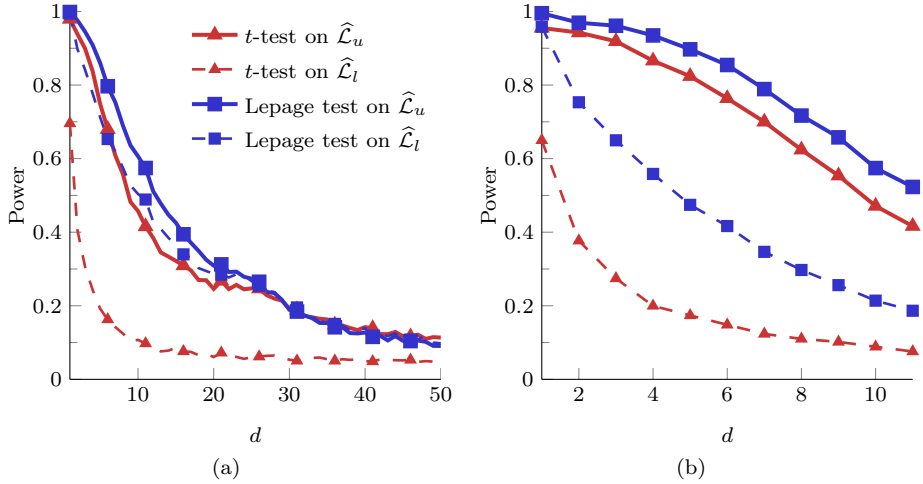

	Figure~\ref{fig:results}(b) shows that the variances of $\l_u$ and $\l_l$ 
	grow linearly with respect to $d$, as in the Gaussian 
	case~\eqref{eq:varince_loglikelihood}. This result indicates that also in this case, where
	datastreams are generated by correlated and bimodal distributions, detectability loss occurs, since the $\SNRPhi$ decreases when $d$ increases. As in Section \ref{subsec:GaussianDatastreams}, we experienced the same trend when the log-likelihoods  $\widehat{\l}_u$ and $\widehat{\l}_l$ are computed with respect to fitted models 
	$\widehat{\phi}_0$. We further observe that $\l_l$ exhibits a much larger 
	variance than $\l_u$, thus we expect this to achieve lower change-detection 
	performance than $\l_u$. Probably, this is why $\l_u$ was 
	used in \cite{Kuncheva2013_TKDE} instead of $\l_l$.

	\subsection{Real-World Data}\label{subsec:RealWorld}
	To investigate detectability loss in real-world datasets, we design a change-detection problem on the \emph{Wine Quality Dataset} \cite{cortez2009modeling} and the \emph{MiniBooNE Particle Dataset} \cite{roe2005boosted} from the UCI repository \cite{Lichman:2013}. The Wine dataset has 12 dimensions: 11 corresponding to numerical results of laboratory analysis (such as density, Ph, residual sugar), and one corresponding to a final grade (from 0 to 10) for each different wine. We consider the vectors of laboratory analysis of all white wines having a grade above 6, resulting in a $11$-dimensional dataset containing 3258 data. The Particle dataset contains numerical measurements from a physical experiment designed to distinguish electron from muon neutrinos. Each sample has $50$-dimensions and we considered only data from muon class, yielding $93108$	 data.
	
	Since in either datasets $\phi_0$ is completely unknown, we need to estimate it for both introducing changes having constant magnitude and computing the log-likelihood. We adopt Gaussian mixtures and estimate $k$ by 5-fold cross validation over the whole datasets, obtaining $k=4$ and $k=2$ for Wine and Particle dataset respectively. 
	
	We process each dataset as follows. Let us denote by $D$ the dataset dimension and for each value of $d = 1, \dots, D$ we consider only $d$ components of our dataset that are randomly selected. We then generate a $d$-dimensional training set of $200\cdot d$ samples and a test set of 1000 samples (datastream), which are extracted by a bootstrap procedure without replacement. The second half of the datastream is perturbed by the change $\widetilde{\phi}_0 \to \widetilde{\phi}_1$, which is defined by fitting at first $\widetilde{\phi}_0$ on the whole $d$-dimensional dataset, and then computing $\widetilde{\phi}_1$ according to the procedure in Appendix. Then, we estimate $\widehat{\phi}_0$ from the training set and we compute $\mathcal{T}(\widehat{W}_P,\widehat{W}_R)$, where $\widehat{W}_P$, $\widehat{W}_R$ are defined as in Section \ref{subsec:GaussianDatastreams}. This procedure is repeated $5000$ times to numerically compute the test power. Note that the number of Gaussians in both $\widetilde{\phi}_0$ and $\widehat{\phi}_0$ is the value of $k$ estimated from whole $D$-dimensional dataset, and that $\widetilde{\phi}_0$ is by no means used for change-detection purposes. 
	
	Figure~\ref{fig:results_real} reports the power of both Lepage and one-sided $t$-tests on the Particle dataset and Wine dataset, considering $\widehat{\l}_u$ \eqref{eq:likelihoodGM_upper} and $\widehat{\l}_l$ \eqref{eq:likelihoodGM_lower} as approximated expressions of the likelihoods. The power of both tests is monotonically decreasing, indicating an increasing difficulty in detecting a change among $\widehat{W}_P$ and $\widehat{W}_R$ when $d$ grows. This result is in agreement with the claim of Theorem \ref{theorem:gaussian_detectability} and the results shown in the previous sections. In contrast of Gaussian datastreams, the Lepage here turns to be more powerful than the $t$-test. This fact indicates that it is important to monitor also the dispersion of  $\l(\cdot)$ in case of Gaussian mixture, where $\l(\cdot)$ can be multimodal. The decay of the power of the Lepage test also indicate that monitoring both expectation and dispersion of $\l(\cdot)$ does not prevent the detectability loss. 
	Figure~\ref{fig:results_real} indicates that $\widehat{\l}_u(\cdot)$ guarantees superior performance than $\widehat{\l}_l(\cdot)$ and this is a consequence of the lower variance of $\widehat{\l}_u(\cdot)$. This fact also underlines the importance of considering the variance of $\l(\cdot)$ in measures of change detectability, as in \eqref{eq:SNRofTheChange}.
	We finally remark that have set a change magnitude ($\KLPhi = 1$) that is quite customary in change-detection experiments, as in the univariate Gaussian case this corresponds to setting $\mb v$ equals to the standard deviation of $\phi_0$. Therefore, since in our experiments the power of both tests is almost halved when $d \approx 10$, we can conclude that detectability loss is not only a Big-Data issue.

	\section{Conclusions}
We provide the first rigorous study of the challenges that change-detection methods have to face when data dimension scales. 
Our theoretical and empirical analyses reveal that the popular approach of monitoring the log-likelihood of a multivariate datastream suffers detectability loss when data dimension increases. 
Remarkably, detectability loss is not a consequence of density-estimation errors -- even though these further reduce detectability -- but it rather refers to an intrinsic limitation of this change-detection approach. 
Our theoretical results demonstrate that detectability loss occurs independently on the specific statistical tool used to monitor the log-likelihood and does not depend on the number of input components affected by the change.
Our empirical analysis, which is rigorously performed by keeping the change-magnitude constant when scaling data-dimension, confirms detectability loss also on real-world datastreams.
Ongoing works concern extending this study to other change-detection approaches and to other families of distributions.

	\section*{Appendix: Generating Changes of Constant Magnitude}\label{appendix}
Here we describe a procedure to select, given $\phi_0$, an orthogonal matrix $Q\in\mR^{d\times d}$ and a vector $\mb v \in \mR^d$ such that $\phi_1 = \phi_0(Q \mb x + \mb v)$ guarantees $\KLPhi=1$ in the case of Gaussian pdfs, and $\KLPhi \approx 1$ for arbitrary distributions. Extensions to different values of $\KLPhi$ are straightforward. Since $\phi_1(\mb x) = \phi_0(Q\mb x+\mb v)$, we formulate the problem as generating at first a rotation matrix $Q$ such that 
$$0 < \textnormal{sKL}(\phi_0(\cdot), \phi_0(Q \cdot)) < 1$$ 
and then defining the 
translation vector $\mb v$ to adjust $\phi_1$ such that $\KLPhi$ reaches (or approaches) 1.

We proceed as follows: we randomly define a rotation axis $\mb r$, and a 
sequence of rotations matrices $\{Q_j\}_j$ around $\mb r$, where the rotation 
angles monotonically decrease toward $0$ (thus $Q_j$ tends to the identity 
matrix as $j \to \infty$). Then, we set $Q_{j^*}$
as the largest rotation yielding a $\textnormal{sKL}<1$, namely
\begin{equation}\label{eq:def_Q_j_star}
	j^* = \min\{j \,:\, \textnormal{sKL}(\phi_0(\cdot), \phi_0(Q_j \cdot)) < 1\}.
\end{equation}
When $\phi_0$ is continuous and bounded (as in case of Gaussian mixtures) it can be easily proved that such a $j^*$ exists. 

In the case of Gaussian pdfs, when 	$\phi_0 = \mathcal{N}(\mu_0,\Sigma_0)$, $\KLPhi$ admits a closed-form expression:
\begin{equation}\label{eq:KL_gaussian}
	\begin{aligned}
		\textnormal{sKL}&(\phi_0,\phi_1) = \frac{1}{2}\biggl[\mb v'\Sigma_0^{-1}\mb v + \mb v'Q\Sigma_0^{-1}Q'\mb v +\\
		&+ 2\mb v'\Sigma_0^{-1}(I-Q)\mu_0 + 2\mb v'Q\Sigma_0^{-1}(Q'-I)\mu_0 +\\ &+\textnormal{Tr}(Q'\Sigma_0^{-1}Q\Sigma_0) + \textnormal{Tr}(\Sigma_0^{-1}Q'\Sigma_0Q)-2d +\\
		&+2\mu_0'(I-Q')\Sigma_0^{-1}(I-Q)\mu_0\biggr]\,,
	\end{aligned}
\end{equation} 
and $\textnormal{sKL}(\phi_0(\cdot), \phi_0(Q_j \cdot))$ can be exactly computed to solve \eqref{eq:def_Q_j_star}. When there are no similar expressions
for $\KLPhi$ this has to be computed via Monte Carlo simulations.

After having set the rotation matrix $Q$, we randomly generate a unit-vector 
$\mb u$ as in \cite{alippi2014intelligence} and determine a suitable 
translation along the line $\mb v = \rho \mb u$, where $\rho >0$, 
to achieve $\KLPhi =1$. Again, the closed-form expression 
\eqref{eq:KL_gaussian} allows to directly compute the exact value of $\rho$ by 
substituting $\mb v = \rho \mb u$ into~\eqref{eq:KL_gaussian}. This yields a 
quadratic equation in $\rho$, whose positive solution $\rho^*$ provides $\mb v 
= \rho^* \mb u$ that leads to $\KLPhi = 1$. When the are no analytical expressions for 
$\KLPhi$, we generate an increasing sequence 
$\{\rho_n\}_n$ such that $\rho_0 = 0$ and $\rho_n\to\infty$ as $n\to\infty$, and set 
\begin{equation}\label{eq:def_rho_star}
	n^* =  \max\{n \,:\, \textnormal{sKL}	(\phi_0(\cdot), \phi_0(Q\cdot + \rho_n\mb u)) < 1\},
\end{equation}
where $\textnormal{sKL}(\phi_0(\cdot), \phi_0(Q\cdot + \rho_n\mb u))$ is computed by Monte 
Carlo simulations. After having solved \eqref{eq:def_rho_star}, we determine 
$\rho^*$ via linear interpolation of $[\rho_{n^*},\rho_{n^*+1}]$ on the 
corresponding values of the sKL. In this case, we can only guarantee $\KLPhi \approx 
1$ with an accuracy that can be improved by increasing the resolution of 
$\{\rho_n\}_n$.
			
	{\bibliographystyle{IEEEtran}
	\bibliography{multivariateCDT}}
\end{document}